%% file: Event2vec_www1.tex
\documentclass[sigconf]{acmart}

\usepackage{booktabs} % For formal tables
\usepackage{subfigure}
\usepackage{amssymb}
\usepackage{amsmath}
\usepackage{amsthm}
\usepackage[ruled,vlined]{algorithm2e}
\usepackage{hyperref}

\setcopyright{rightsretained}
\acmDOI{10.475/123_4}

\acmISBN{123-4567-24-567/08/06}

\acmConference[WOODSTOCK'97]{ACM Woodstock conference}{July 1997}{El
	Paso, Texas USA}
\acmYear{1997}
\copyrightyear{2016}

\acmArticle{4}
\acmPrice{15.00}

\begin{document}
\title{Representation Learning for Heterogeneous Information Networks via Embedding Events}

\author{Guoji Fu}
\affiliation{
	\institution{Southern University of Science and Technology}
	\city{Shenzhen}
	\country{China}
}
\email{fuguoji1995@gmail.com}

\author{Bo Yuan}
\affiliation{
	\institution{Southern University of Science and Technology}
	\city{Shenzhen}
	\country{China}
}
\email{yuanb@sustc.edu.cn}

\author{Qiqi Duan}
\affiliation{
	\institution{Southern University of Science and Technology}
	\city{Shenzhen}
	\country{China}
}
\email{duanqq257@qq.com}

\author{Xin Yao}
\affiliation{
	\institution{Southern University of Science and Technology}
	\city{Shenzhen}
	\country{China}
}
\email{xiny@sustc.edu.cn}

\begin{abstract}
Network representation learning (NRL) has been widely used to help analyze large-scale networks through mapping original networks into a low-dimensional vector space. However, existing NRL methods ignore the impact of properties of relations on the object relevance in heterogeneous information networks (HINs). To tackle this issue, this paper proposes a new NRL framework, called Event2vec, for HINs to consider both quantities and properties of relations during the representation learning process. Specifically, an event (i.e., a complete semantic unit) is used to represent the relation among multiple objects, and both event-driven first-order and second-order proximities are defined to measure the object relevance according to the quantities and properties of relations. We theoretically prove how event-driven proximities can be preserved in the embedding space by Event2vec, which utilizes event embeddings to facilitate learning the object embeddings. Experimental studies demonstrate the advantages of Event2vec over state-of-the-art algorithms on four real-world datasets and three network analysis tasks (including network reconstruction, link prediction, and node classification).
\end{abstract}



\maketitle

\input{Event2vecbody-conf}

\bibliographystyle{ACM-Reference-Format}
\bibliography{bibliography}

\end{document}

%% file: Event2vecbody-conf.tex
\section{Introduction}
Heterogeneous information networks (HINs), which contain multiple types of objects and links, are ubiquitous in a variety of real-world scenarios such as social networks \cite{jiang2017semi}, bibliographic networks \cite{sun2011pathsim}, and user interest networks \cite{chen2018heterogeneous}. Many real-world HINs are large-scale, e.g., social networks with millions of nodes \cite{ellison2007benefits}. To analyze large-scale networks in many applications, such as link prediction \cite{zhang2013predicting,zhang2015diffusion}, node classification \cite{ji2011ranking,jacob2014learning}, effective network analysis techniques are needed. However, most network analysis methods suffer from high computation and space cost \cite{cai2018comprehensive}. To tackle this problem, a mainstream idea is network representation learning (NRL), which maps original networks into a low-dimensional vector space while preserving as much of the original network information as possible. Using the low-dimensional vector representations of objects as input features, the performance of downstream network analysis can be improved \cite{cai2018comprehensive}. Due to the heterogeneities of both objects and relations, the primary challenge of NRL for HINs is that the representation learning process should effectively capture original network structural and semantic information. To this end, this paper aims to propose an effective NRL framework to learn object embeddings for HINs.

\begin{figure}
	\centering
	\includegraphics[scale=0.25]{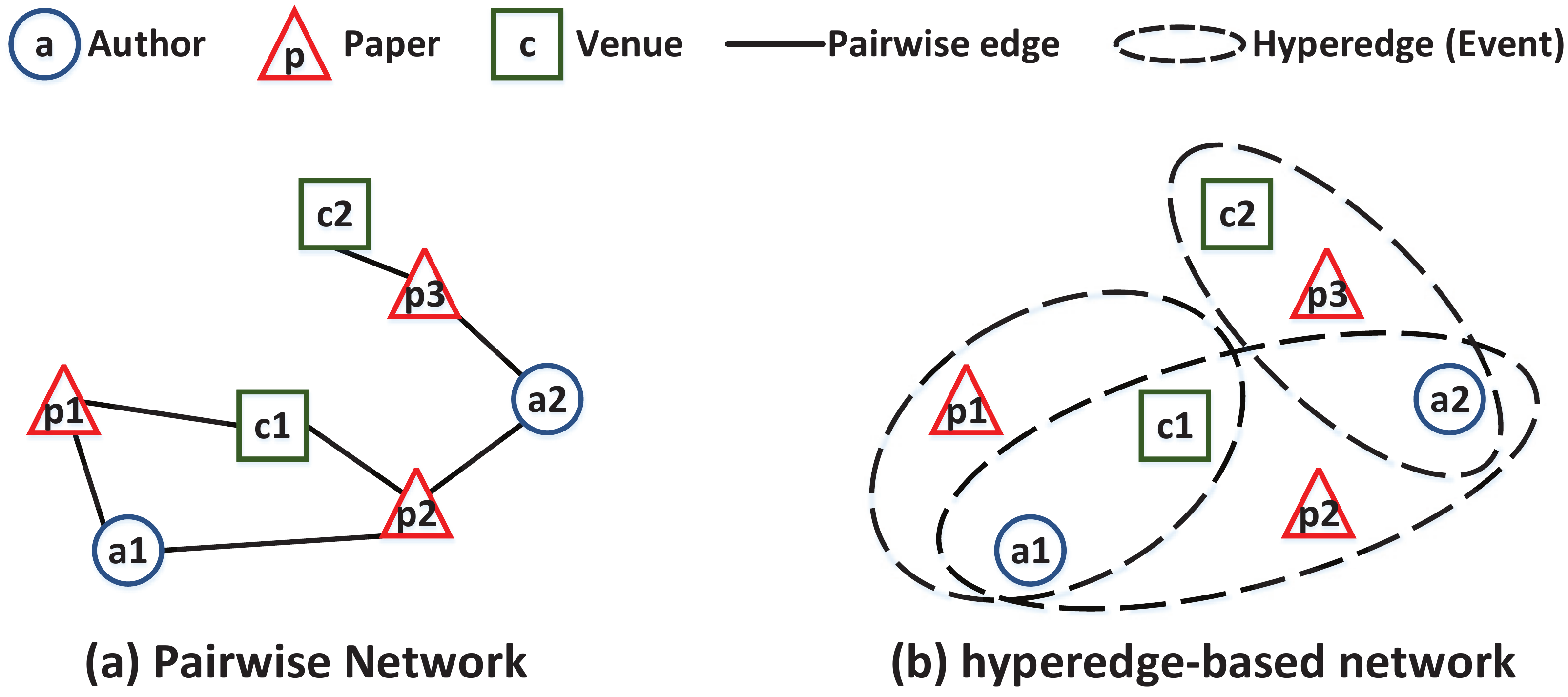}
	\caption{{\small Example of a bibliographic network to illustrate the difference between pairwise and hyperedge-based networks.}}\label{figure1}
\end{figure}

In real-world networks, there may exist some relations among multiple objects. Taking Figure \ref{figure1} as an example, the relation among authors, a paper, and a venue is an indecomposable unit. Decomposing it into pairwise object relations will lose some semantic information \cite{tu2017structural}. Recently, {\itshape hyperedge} was used to represent the relation among multiple objects \cite{gui2017embedding,tu2017structural}, which can be regarded as a complete semantic unit called {\itshape event} \cite{gui2017embedding}. These hyperedge-based methods measure the relation among multiple objects as a whole. However, they consider only the quantities of events and ignore their properties during the representation learning process.

Intuitively, objects involved in same events should be relevant. Meanwhile, objects involved in similar events should be relevant as well. As examples of two bibliographic networks illustrated in Figure \ref{figure2}. In Figure \ref{figure2}(a), author $a_1$ and $a_2$ published two papers together, they are both involved in event $e_1$ and $e_2$. Therefore, they are relevant. Figure \ref{figure2}(b) shows author $a_1$ and $a_2$ published papers with the same topic in the same venue. They are involved in two similar events $e_1$ and $e_2$, respectively. Hence, they should be relevant as well. The properties of events can facilitate capturing the semantic relevance among objects. The relevance among objects in HINs should be driven by both the number of their intersectional events (event-driven first-order proximity) and the similarity between their events (event-driven second-order proximity).

\begin{figure}
	\centering
	\includegraphics[scale=0.26]{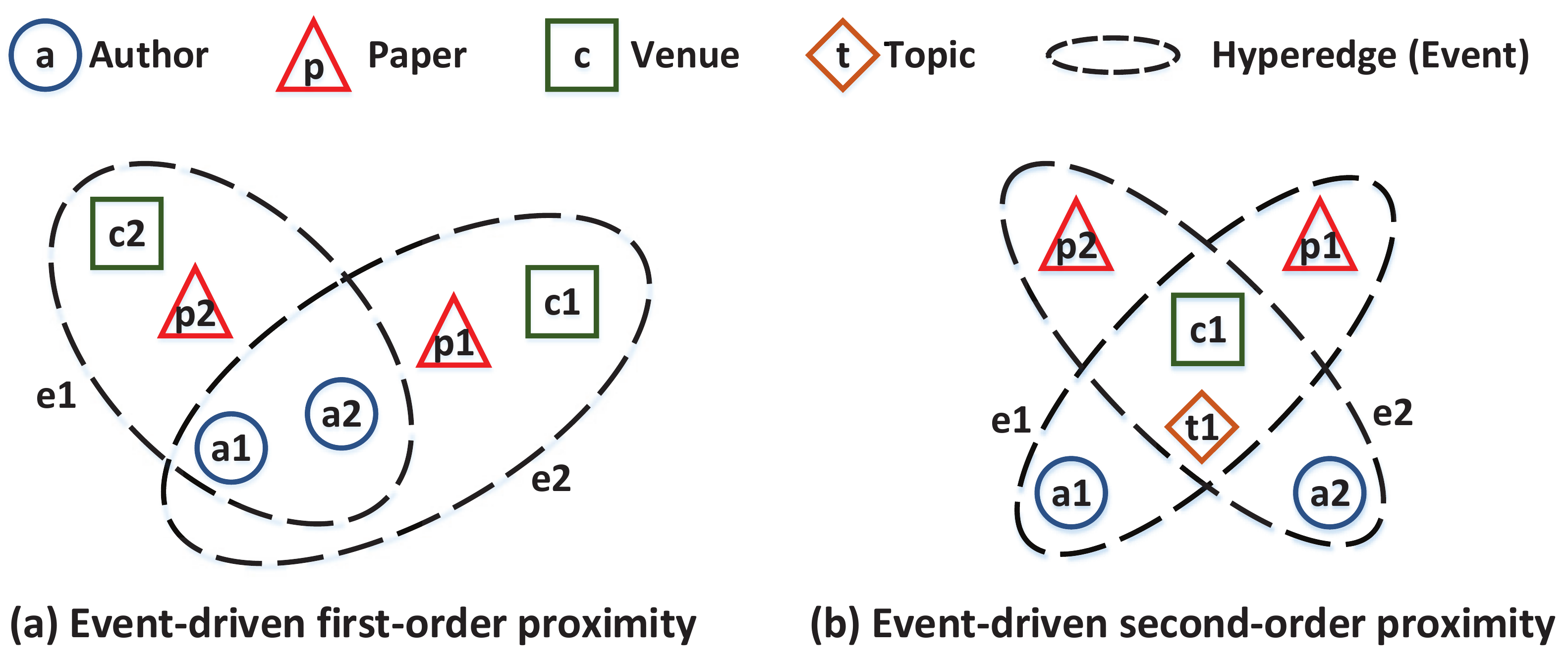}
	\caption{Examples of event-driven proximity of bibliographic networks}\label{figure2}
\end{figure}

In this paper, events are used to represent the relations among objects, and both event-driven first-order and second-order proximities are used to measure the object relevance according to the quantities and properties of relations. We propose a new NRL framework, called Event2vec, to learn the object embeddings of HINs via two learning steps. The first step uses an autoencoder to learn the event embeddings. Based on the event embeddings learned by the previous step, the second step obtains the object embeddings by preserving the event-driven proximities. We theoretically prove the leaning process which utilizes event embeddings to facilitate learning object embeddings is capable to preserve the event-driven proximities in the embedding space.

The contributions of this paper are summarised as follows:

\begin{itemize}
	\item We investigate the significance of properties of relations among multiple objects for learning HIN representations.
	\item We define the event-driven first-order and second-order proximities to measure object relevance driven by quantities and properties of relations, respectively.
	\item We propose a new NRL framework called Event2vec to learn the object embeddings of HINs and theoretically prove Event2vec can preserve the event-driven first-order and second-order proximities in the embedding space.
	\item Experiments on four real-world datasets and three network analysis tasks are conducted to demonstrate the effectiveness of Event2vec.
\end{itemize}

The rest of this paper is organized as follows. Section 2 reviews the related work. In Section 3, we give the definition of the problem. The details of the proposed framework are given in Section 4. Section 5 presents the experimental results. Finally, we conclude in Section 6.

\begin{figure*}
	\centering
	\subfigure[{\small Input Network}]{
		\includegraphics[scale=0.28]{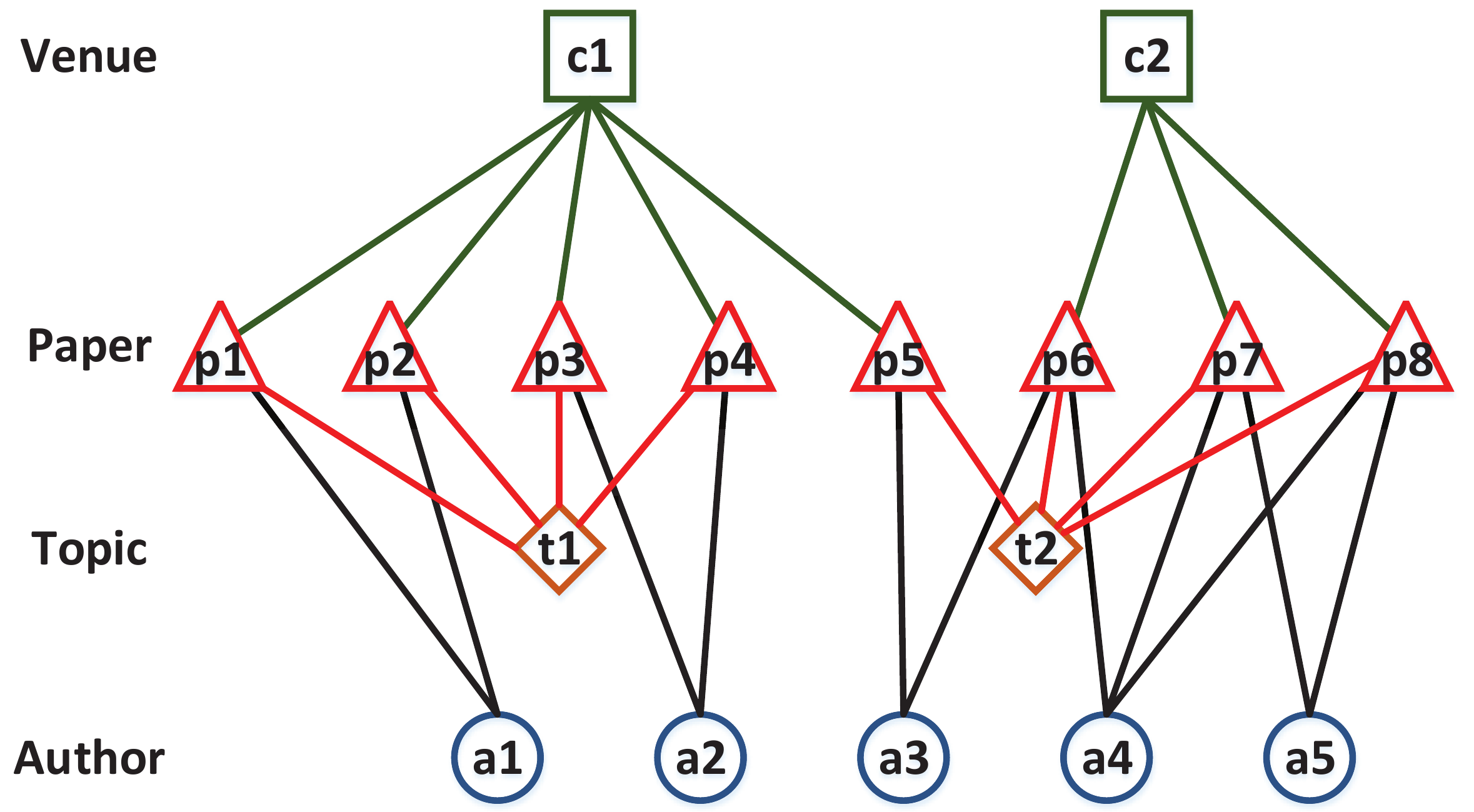}
	}
	\subfigure[{\small Network Representation Learning Results}]{
		\includegraphics[scale=0.35]{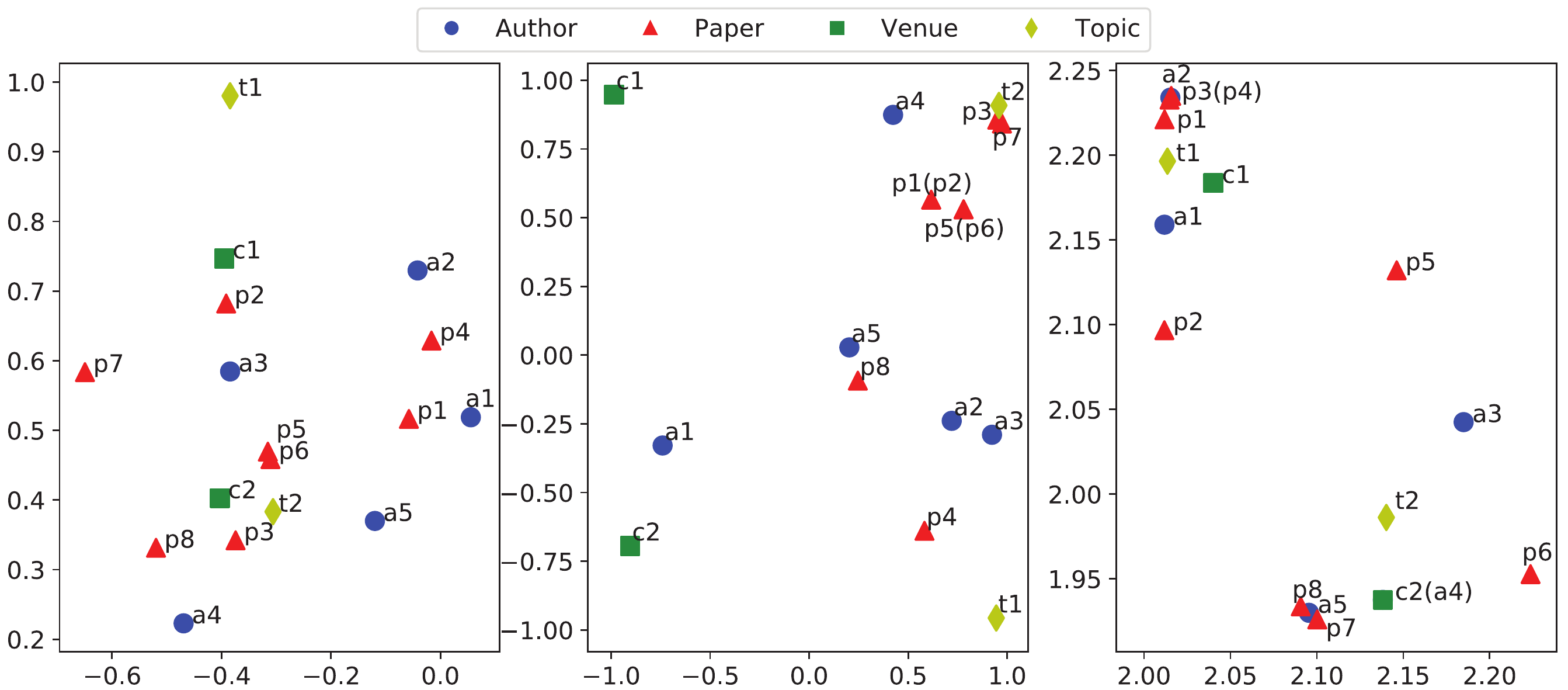}
	}
	\caption{{\small Example of using DeepWalk((b) left), DHNE((b) middle) and Event2vec((b) right) to learn the 2-dimensonal object embeddings. In the embedding space, $a_3$, $a_4$, and $a_5$ should be close since they have cooperations in publishing papers; $a_1$ and $a_2$ published same topic papers in the same venue, they have the semantic relevance and hence should be close as well. Event2vec obtains the best performance in capturing the above object relevance. The properties of relations facilitate capturing the object relevance.}}\label{figure3}
\end{figure*}

\section{Related Work}
The related work is in the area of network representation learning. Early works in NRL community were mainly designed for homogeneous information networks, which contain only a single type of objects and links. However, HINs are more ubiquitous in most complex real-world scenarios. Recently, representation learning for HINs has attracted increasing interest in the NRL community. We review the works in representation learning for homogeneous information networks and heterogeneous information networks in the following tow sub-sections, respectively.

\subsection{Homogeneous Information Network Representation Learning}
Many works have been proposed to learn representations of homogeneous information networks. They can be classified into matrix factorization-based methods, probability-based methods, and deep learning-based methods. Matrix factorization-based methods \cite{roweis2000nonlinear,belkin2002laplacian,ahmed2013distributed} represent relations between pairwise objects in the form of a matrix, e.g., adjacent matrix, Laplacian matrix, and factorize the matrix to obtain object embeddings using eigen-decomposition. Probability-based methods such as DeepWalk \cite{perozzi2014deepwalk} and node2vec \cite{grover2016node2vec} use random walks to sample paths from the network and calculate object co-occurrence probabilities which are used to learn object embeddings via the Skip-gram model \cite{mikolov2013efficient,mikolov2013distributed}. LINE \cite{tang2015line} preserves both first-order and second-order proximities of networks by minimizing the Kullback-Leibler divergence of two joint probability distributions for each pair nodes. Deep learning-based methods \cite{cao2016deep,wang2016structural,kipf2016semi} use adjacent matrix, or object co-occurrence probability matrix, or graph convolution as input, and learn object embeddings via a deep neural network.

Those methods are effective to capture the structural and semantic information of homogeneous information networks. However, they fail to capture the complete semantic information of HINs since they ignore the different semantics of relations among different types of objects.

\subsection{Heterogenous Information Network Representation Learning}
Researchers in NRL field have increasingly engaged in HIN representation learning recently. The success of applying metapath \cite{sun2011pathsim} in HIN analysis has motivated some researchers to carry out metapath-based methods to learn representations of HINs. Metapath2vec \cite{dong2017metapath2vec} and HIN2Vec \cite{fu2017hin2vec} extend the Skip-gram model to learn the embeddings of HINs by employing the metapath-based random walks. HINE \cite{huang2017heterogeneous} optimizes the defined objective function which aims to preserve the metapath-based proximities. However, they only consider the relations between pairwise objects. In order to capture the complete semantics of relations among multiple objects, hyperedges have been used to represent the relations among objects. HEBE \cite{gui2017embedding} preserves the proximites of the objects by modeling the relations among objects as hyperedges. DHNE \cite{tu2017structural} is a hyperedge-based method that preserves both first-order and second-order hypergraph structural information through a semi-supervised neural network model.

The aforementioned methods consider only the number of relations among objects while overlooking the impact of their properties. However, the properties of relations are important for NRL to capture the semantic information of HINs. On the contrary, Event2vec is able to consider both quantities and properties of relations. Given the bibliographic network shown in Figure \ref{figure3}(a), Figure \ref{figure3}(b) shows the results of representation learning using DeepWalk, DHNE, and Event2vec which demonstrate that the properties of relations among multiple objects can facilitate capturing the original network structural and semantic information.

\section{Problem Definition}
In this section, we formally define the problem of representation learning for HINs. Firstly, we give the definition of HIN as presented below.
\smallskip

\newtheorem{mydef}{Definition}
\begin{mydef}
	\textbf{(Heterogeneous Information Network \cite{shi2017survey})}. Given an information network $\mathcal{G}=(V, E, T)$, where $V$ is a set of vertexes, $E$ is a set of links, and $T$ is a set of object types and link types. Let $\varphi(v):V\rightarrow T_V\subset T$ be an object type mapping function and $\psi(r): E\rightarrow T_E\subset T$ be a link mapping function. If $|T_V|+|T_E|>2$, we say that $\mathcal{G}$ is a heterogeneous information network. Note that if $|T_V|+|T_E|=2$, it is degraded to a homogeneous information network.
\end{mydef}

Figure \ref{figure1}(a) gives an example of HINs, i.e., a tiny bibliographic network containing three types of objects ({\itshape author}, {\itshape paper}, and {\itshape venue}).

\begin{mydef}
	\textbf{(Events \cite{gui2017embedding})}. An event $e \in \Omega$ is an indecomposable unit formed by a set of objects, representing the consistent and complete semantic information of relation among multiple objects. $\Omega_i$ denotes the set of events that contain object $v_i$.
\end{mydef}

As shown in Figure \ref{figure2}(a), the relation among $a_1$, $a_2$, $p_1$, and $c_1$ is a complete semantic unit, denoted as an event $e_2$.

\begin{mydef}
	\textbf{(Incident Matrix)}. An incident matrix $\mathcal{H}_{|V|\times|\Omega|}$ is a matrix that shows the relationship between objects $V$ and events $\Omega$ in which each row represents an object and each column represents an event. If object $v_i$ belongs to event $e_j$, then $\mathcal{H}_{i,j}=1$, otherwise $\mathcal{H}_{i,j}=0$. Given an HIN with $|T_V|$ types of objects, there are $|T_V|$ incident matrices $\{\mathcal{H}^t\}_{t=1}^{|T_V|}$ in which each $\mathcal{H}_t$ represents the relationship among $t$-th type of objects and all events.
\end{mydef}

\begin{mydef}
	\textbf{(Event-driven First-order Proximity (EFP))}. The event-driven first-order proximity of object $v_i$ and $v_j$ is defined to be the ratio of the number of their intersectional events and the number of their unioned events: 
	\begin{equation}\label{equation1}
	s_{i,j}^1 = \frac{|\Omega_i\cap\Omega_j|}{|\Omega_i\cup\Omega_j|}.
	\end{equation}
\end{mydef}

In Figure \ref{figure2}(a), $a_1$ and $a_2$ have the EFP since they are contained in two same events $e_1$ and $e_2$, $s^{1}_{1,2}=2/2=1$. EFP considers the relevance among objects driven by the quantities of their relations. Larger $s_{i,j}^1$ of object $v_i$ and object $v_j$ indicates their stronger EFP. As a result, $v_i$ and $v_j$ should be closer in the embedding space.

\begin{mydef}
	\textbf{(Event-driven Second-order Proximity (ESP))}. The event-driven second-order proximity of object $v_i$ and $v_j$ is defined to be the average cosine similarity of their non-intersectional events:
	\begin{equation}\label{equation2}
	\begin{aligned}
	s_{i,j}^2 &= sim(\Omega_i, \Omega_j) \\
	&= \frac{1}{|\Omega_i\cup\Omega_j|}\sum_{e\in\Omega_i,k\in\Omega_j,e\neq k}sim(e, k).
	\end{aligned}
	\end{equation}
	\noindent where $sim(e, k)$ denotes the cosine similarity between $e$ and $k$.
\end{mydef}

In Figure \ref{figure2}(b), $a_1$ and $a_2$ have the ESP since $e_1$ and $e_2$ are similar, $s^{2}_{1,2}=\frac{1}{2}sim(e_1, e_2)$. ESP considers the relevance among objects driven by the properties of relations. Larger $s_{i,j}^2$ of object $v_i$ and object $v_j$ indicates their stronger ESP. Therefore, $v_i$ and $v_j$ should be closer in the embedding space.

\begin{mydef}
	\textbf{(Heterogeneous Information Network Representation Learning)}. Given an HIN $\mathcal{G}=(V, E, T)$, HIN representation learning aims to learn a mapping function $f:V \rightarrow Z \in \mathbb{R}^d$ where $d \ll |V|$, and preserve both event-driven first-order and second-order proximities of objects in the embedding space $\mathbb{R}^d$.
\end{mydef}

\section{Event2vec}
In this section, we introduce the proposed Event2vec to learn the object embeddings of HINs. As shown in Figure \ref{figure4}, after generating events from the input pairwise-based HIN, the representation learning process of Event2vec consists of two steps. The first step tries to learn the event embeddings. The second step obtains the object embeddings based on the learned event embeddings.

\begin{figure*}
	\centering
	\includegraphics[scale=0.55]{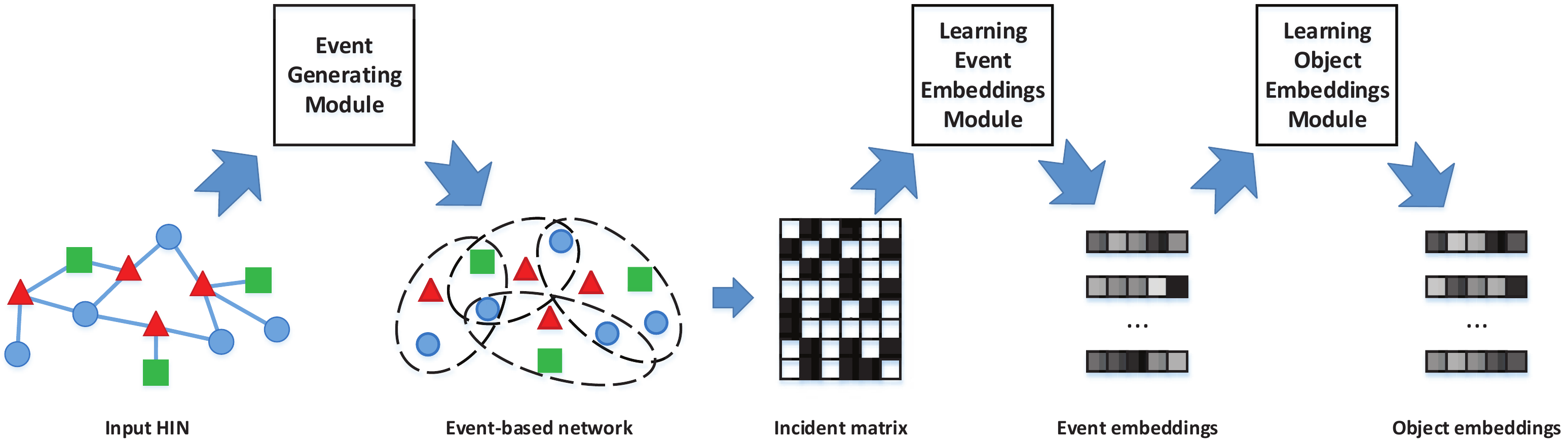}
	\caption{{\small The framework of Event2vec.}}\label{figure4}
\end{figure*}

\subsection{Event Generating}
In this section, we introduce the event generating algorithm. We first define the event identifier $q$ for each link $r$ of the original HIN by defining the mapping function $\Theta(r): r\rightarrow q$. Then the links with same event identifier are merged into an event $e$.

The event identifier $q$ for event $e$ can be defined as a sub-set objects of $e$. The choices of those objects are based on the characteristics of the network in question. For example, for a bibliographic network that contains three types of objects ({\itshape author}, {\itshape paper}, and {\itshape venue}), an incident that authors published a paper in a venue is an indecomposable semantic unit. Therefore, the relation among authors, a paper, and a venue should be regarded as an event. Then the event identifier can be defined as every paper since one event just corresponds to one paper.

Once the event identifiers are defined, events can be generated using Algorithm \ref{algorithm1}.

\begin{algorithm}
	\DontPrintSemicolon
	\KwIn{HIN $\mathcal{G} = (V, E, T)$, Event identifier mapping function $\Theta$}
	\KwOut{Events set $\Omega$}
	\Begin{
		$\Omega \longleftarrow \emptyset$;\\
		\For{$r \in E$}{
			$q \longleftarrow \Theta(r)$;\\
			
			\For{$e \in \Omega$}{
				\If{$\Theta(e) = q$}{
					$e \longleftarrow e \cup r$;\\
					update $\Omega$ using $e$;\\
					break;
				}
			}
			\If{$\Omega$ is $\emptyset$ or $\Theta(e) \neq q, \forall e\in\Omega$}{
				$\Omega \longleftarrow \Omega \cup r$;
			}
		}
	}
	\caption{Event Generating}\label{algorithm1}
\end{algorithm}

\subsection{Learning Event Embeddings}
We use an autoencoder model to learn the event embeddings. It is worth to note that other methods such as singular value decomposition \cite{golub1970singular}, stacked denoising autoencoders \cite{vincent2010stacked}, variational autoencoder \cite{doersch2016tutorial}, can also be adapted to learn the event embeddings. 

\begin{figure}
	\centering
	\includegraphics[scale=0.35]{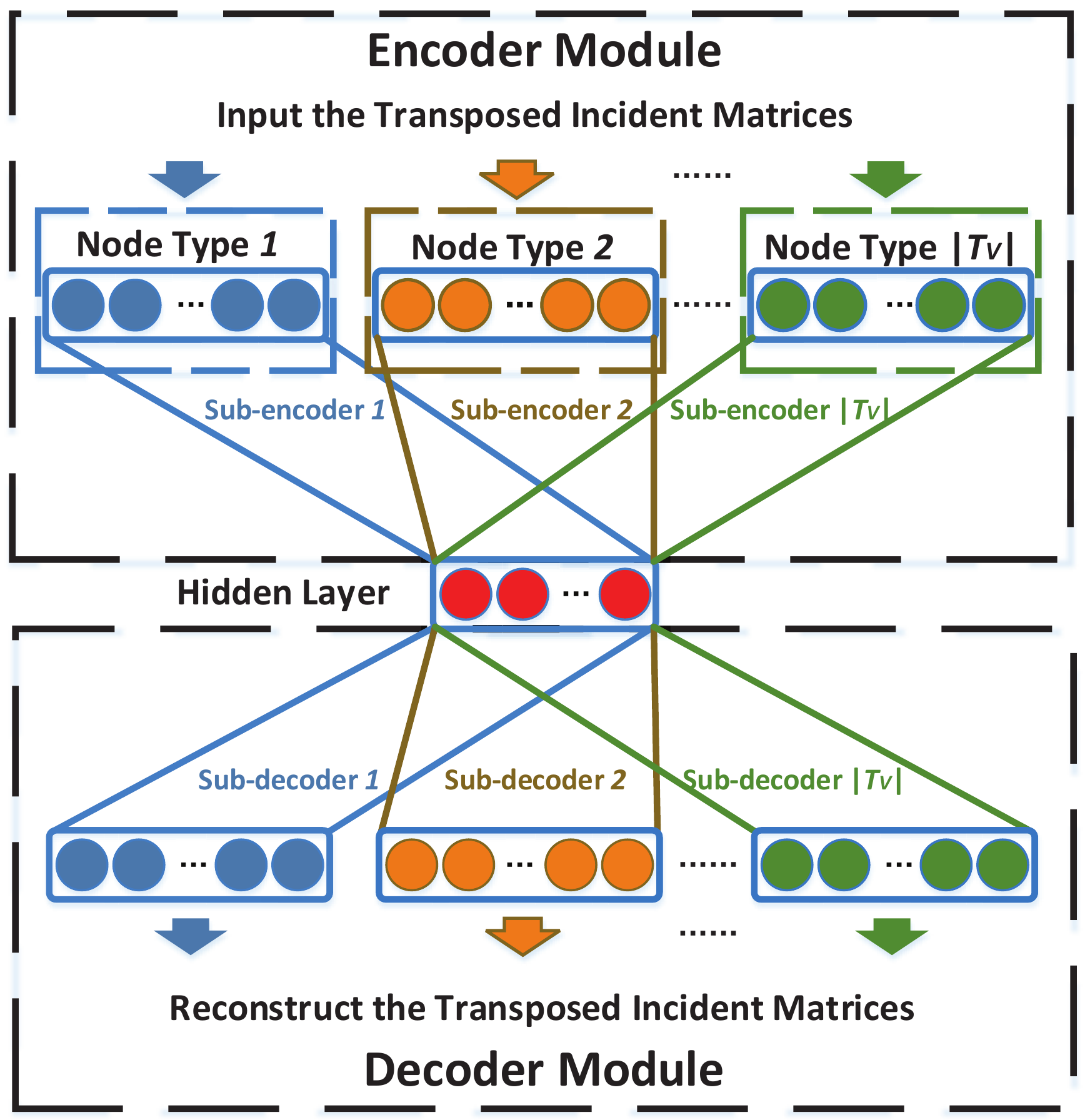}
	\caption{{\small The framework of autoencoder.}}\label{figure5}
\end{figure}

As shown in Figure \ref{figure5}, different types of objects require learning different mapping functions to map events into the same embedding space. Taking the incident matrices $\{\mathcal{H}^t\}_{t=1}^{|T_V|}$ as inputs, we design a single hidden layer autoencoder model with $|T_V|$ sub-encoder components and $|T_V|$ sub-decoder components. Each component is designed for one type of object. The size of the input layer of each sub-encoder component is equal to the number of objects with the same type. The same scheme works for the output layer of each sub-decoder component.

The encoder module maps the incident matrices $\{\mathcal{H}^t\}_{t=1}^{|T_V|}$ from all sub-encoder components to the same embedding space. Then, the embedding of each event $\textbf{Z}_i$ is obtained in the hidden layer by summing up the event embeddings $\{\textbf{Z}_i^t\}_{t=1}^{|T_V|}$ generated from sub-encoder components:
\begin{equation}\label{equation3}
\textbf{Z}_i^t = \sigma(\textbf{W}_t*\mathcal{H}_i^t + \textbf{b}_t),
\end{equation}

\begin{equation}\label{equation4}
\textbf{Z}_i = \sum_{t=1}^{|T_V|}\textbf{Z}_i^t,
\end{equation}

\begin{equation}\label{equation5}
\hat{\mathcal{H}_i^t} = \sigma(\hat{\textbf{W}}_t*\textbf{Z}_i + \hat{\textbf{b}}_t),
\end{equation}

\noindent where $t$ is the index for object types, $\sigma$ is the sigmoid function, $\textbf{W}$ and $\textbf{b}$ are weights and bias of encoder modules and $\hat{\textbf{W}}$ and $\hat{\textbf{b}}$ are weights and bias of decoder modules.

\smallskip
The training objective of the autoencoder is to minimize the reconstruction error between inputs and outputs:

\begin{equation}\label{equation6}
\mathcal{L}_{rec} = \sum_{i=1}^{|\Omega|}\sum_{t=1}^{|T_V|}{||\mathcal{H}_i^t-\hat{\mathcal{H}}_i^t||_2^2}
\end{equation}

We may not observe all links of the HINs. Therefore, it is more meaningful for the autoencoder to reconstruct the non-zero elements of incident matrices correctly. Hence, we impose more penalty to the reconstruction error of the non-zero elements, following the suggestions from the previous paper \cite{wang2016structural}. The reconstruction error function is revised as below.

\begin{equation}\label{equation7}
\begin{aligned}
\mathcal{L}_{rec}^* &= \sum_{i=1}^{|\Omega|}\sum_{t=1}^{|T_V|}{||(\mathcal{H}_i^t-\hat{\mathcal{H}}_i^t)\odot\Lambda_i^t||_2^2}\\
&=\sum_{t=1}^{|T_V|}{||(\mathcal{H}^t-\hat{\mathcal{H}}^t)\odot\Lambda^t||_2^2},
\end{aligned}
\end{equation}

\noindent where $\odot$ is the Hadamard product, $\Lambda_i^t = \{\lambda_{i,j}^t\}_{j=1}^{|V^t|}$. We define\\
\[ \lambda_{i,j}^t =
\begin{cases}
1, & \mbox{if } \mathcal{H}_{ij}^t = 0,\\
\beta, & \mbox{otherwise}.
\end{cases}
\]

\noindent with $\beta>1$.

\smallskip
To avoid overfitting, we add the $\mathcal{L}2$-norm regularizer term to the loss function $\mathcal{L}_{reg}$ where \\
\begin{equation}\label{equation8}
\mathcal{L}_{reg} = \sum_{t=1}^{|T_V|} (||\textbf{W}_t||_2^2 + ||\hat{\textbf{W}}_t||_2^2).
\end{equation}

\noindent Then the final objective function is shown as follows:

\begin{equation}\label{equation9}
\mathcal{L} = \mathcal{L}_{rec}^* + \alpha\mathcal{L}_{reg},
\end{equation}

\noindent where $\alpha \ge 0$ is the penalty rate.
\smallskip

The stochastic gradient descent algorithm is used to train the proposed autoencoder. The partial derivatives of $\mathcal{L}$ with respect to $\textbf{W}$ and $\hat{\textbf{W}}$ can be calculated as follows:

\begin{equation}\label{equation10}
\begin{aligned}
\frac{\partial \mathcal{L}}{\hat{\textbf{W}}_t} &= \frac{\partial \mathcal{L}_{rec}^*}{\partial \hat{\textbf{W}}_t} + \alpha\frac{\partial \mathcal{L}_{reg}}{\partial \hat{\textbf{W}}_t}, \\
\frac{\partial \mathcal{L}}{\textbf{W}_t} &= \frac{\partial \mathcal{L}_{rec}^*}{\partial \textbf{W}_t} + \alpha\frac{\partial \mathcal{L}_{reg}}{\partial \textbf{W}_t},
\end{aligned}
\end{equation}
\smallskip

The above derivatives of Eq.\ref{equation10} are obtained iteratively using the back-propagation algorithm (BP) \cite{lecun2015deep} during the training process.
\smallskip

After training the model, the event embeddings can be obtained in the hidden layer by repeating the encoder process. Events that contain many identical objects will obtain similar embeddings using autoencoder.

\subsection{Learning Object Embeddings}
The object embeddings are obtained based on the event embeddings learned from the above sub-section. As previously discussed, large $s_{i,j}^1$ and $s_{i,j}^2$ both deduce object $v_i$ and $v_j$ are relevant. To capture the relevance of objects, we need to preserve both their EFP and ESP. Based on this motivation, we obtain the object embedding $y_i$ by taking the average of event embeddings $\{z_e\}, e\in \Omega_i$ as below.

\begin{equation}\label{equation11}
y_i = \frac{1}{|\Omega_i|}\sum_{e\in\Omega_i}z_e.
\end{equation}

\noindent Rewriting the above equation into matrix form as below, which we obtain $\textbf{Y}^t$,

\begin{equation}\label{equation12}
\textbf{Y}^t = (\textbf{D}_v^t)^{-1} \mathcal{H}^t \textbf{Z}^t,
\end{equation}

\noindent where $\textbf{D}_v^t$ is the diagonal matrix that contains the $t$-th type of object degrees.
\smallskip

Through the above learning process, the learned object embeddings satisfy the following lemma,
\newtheorem{myprop}{Lemma}
\begin{myprop}
	The similarity between embedding $y_i$ of object $v_i$ and embedding $y_j$ of $v_j$ is proportional to their event-driven first-order proximity $s_{i,j}^1$ and second-order proximity $s_{i,j}^2$.
\end{myprop}
\begin{proof}
	Given object $v_i$ and $v_j$, the cosine similarity between their embeddings $y_i$ and $y_j$ is shown as below.
	\begin{equation}\label{equation13}
	\begin{aligned}
	sim(y_i, y_j) = {} & sim(\frac{1}{|\Omega_i|}\sum_{e\in\Omega_i}z_e,  \frac{1}{|\Omega_j|}\sum_{k\in\Omega_j}z_k) \\
	\propto {} & \sum_{e\in\Omega_i,k\in\Omega_j}sim(z_e, z_k) \\
	\propto {} & \frac{1}{|\Omega_i\cup\Omega_j|}(\sum_{e\in\Omega_i,k\in\Omega_j,e=k}sim(z_e, z_k) \\ 
	& + \sum_{e\in\Omega_i,k\in\Omega_j,e\neq k}sim(z_e, z_k))  \\
	= {} & \frac{|\Omega_i\cap\Omega_j|}{|\Omega_i\cup\Omega_j|} + sim(\Omega_i, \Omega_j) \\
	= {} & s_{i,j}^1 + s_{i,j}^2.
	\end{aligned}
	\end{equation}
\end{proof}
Hence, objects that have many intersectional events and$/$or similar events will obtain similar embeddings. Specifically, the larger $s_{i,j}^1$ and$/$or $s_{i,j}^2$ of objects $v_i$ and $v_j$, the more similar embeddings they have. Therefore, Event2vec is able to preserve both EFP and ESP of objects in the embedding space.
\smallskip

Finally, we present the Event2vec in Algorithm \ref{algorithm2}.

\begin{algorithm}
	\DontPrintSemicolon
	\KwIn{The incident matrix $\mathcal{H}$, the object degrees diagonal matrix $\{D_v^t\}_{t=1}^{|T_V|}$ and the parameter $\alpha$}
	\KwOut{Event embeddings $\textbf{Z}$, object embeddings $\textbf{Y}$}
	\Begin{
		Initialize $\{\textbf{W}\}_{t=1}^{|T_V|}$, $\{\hat{\textbf{W}}\}_{t=1}^{|T_V|}$, $\{\textbf{b}\}_{t=1}^{|T_V|}$, $\{\hat{\textbf{b}}\}_{t=1}^{|T_V|}$;\\
		\Repeat{the terminating condition is met}{
			Calculate $\textbf{Z}$ by Eq.\ref{equation3} and Eq.\ref{equation4};\\
			Calculate $\hat{\mathcal{H}}$ by Eq.\ref{equation5};\\
			Calculate objective function cost by Eq.\ref{equation9};\\
			Update $\{\textbf{W}\}_{t=1}^{|T_V|}$, $\{\hat{\textbf{W}}\}_{t=1}^{|T_V|}$, $\{\textbf{b}\}_{t=1}^{|T_V|}$, $\{\hat{\textbf{b}}\}_{t=1}^{|T_V|}$ using BP based on Eq.\ref{equation10};\\
		}
		Calculate $\textbf{Y}$ by Eq.\ref{equation12};
	}
	\caption{Event2vec}\label{algorithm2}
\end{algorithm}

\begin{table*}
	\centering
	\caption{{\small Description of four datasets.}}\label{table1}
	\begin{tabular}{|c|c|c|c|c|c|c|c|c|c|}
		\toprule
		Datasets & \multicolumn{4}{|c|}{Object type} & \multicolumn{4}{|c|}{\#(V)} & \#(E) \\
		\midrule
		DBLP & author & paper & venue & term & 14475 & 14376 & 20 & 8920 & 170794 \\
		Douban & user & movie & actor & director & 3022 & 6977 & 3004 & 789 & 214392 \\
		IMDB & user & movie & actor & director & 943 & 1360 & 42275 & 918 & 136093\\
		Yelp & user & business & location & category & 14085 & 14037 & 62 & 575 & 247698 \\
		\bottomrule
	\end{tabular}
\end{table*}

\subsection{Complexity Analysis}\label{complexity}
The computational complexity of event generating is $O(|E||\Omega|)$, where $|E|$ is the number of links in HIN and $|\Omega|$ is the number of generated events. The computational complexity of training autoencoder is $O(|V|dbI)$, where $|V|$ is the number of objects in HIN, $d$ is the representation size, $b$ is the batch size and $I$ is the number of iterations. The computational complexity of generating object embeddings is $O(|V||\Omega|d)$. Then the total computational complexity of Event2vec is $O(|E||\Omega|)+O(|V|dbI)+(|V||\Omega|d)$.

\section{Experiments}
This section reports experimental results of Event2vec. We use four real-world datasets to evaluate our method on three network analysis tasks including network reconstruction, link prediction, and node classification. Specifically, network reconstruction task is used to evaluate the performance of NRL methods for preserving structural information and link prediction and node classification are used to evaluate the performance of simultaneously preserving both the original network structural and semantic information. The source code of Event2vec is available at {\itshape \url{https://github.com/fuguoji/event2vec}}.

\subsection{Datasets}
We evaluate our method on four real-world datasets, including DBLP \cite{sun2011pathsim}, Douban \cite{zheng2017recommendation}, IMDB\footnote{http://komarix.org/ac/ds/} and Yelp\footnote{https://www.yelp.com/}. The brief information of each dataset is shown as follows.

\begin{itemize}
	\item \textbf{DBLP}: DBLP is a bibliographic network in computer science collected from four research areas: database, data mining, machine learning, and information retrieval. In the dataset, 4057 authors, 20 venues and 100 papers are labeled with one of the four research areas.
	\item \textbf{Douban}: Douban was collected from a user review website Douban in China. We extracted a sub-network containing four types of objects for our experiments.
	\item \textbf{IMDB}: IMDB is a link dataset collected from the Internet Movie Data. The network used in the experiment contains four types of objects. In the dataset, 1357 movies are labeled with at least one of the 23 labels.
	\item \textbf{Yelp}: The dataset was extracted from a user review website in America, Yelp, containing four types of objects.
\end{itemize}

All four datasets are used in network reconstruction and link prediction tasks, but only DBLP and IMDB are used in the node classification task since only these two datasets provide the ground truth of object labels. The detailed statistics of datasets are shown in Table \ref{table1}.

\subsection{Baseline Algorithms}
We compare our method with five state-of-the-art methods, including DeepWalk \cite{perozzi2014deepwalk}, node2vec \cite{grover2016node2vec}, LINE \cite{tang2015line}, metapath2vec\cite{dong2017metapath2vec}, and DHNE \cite{tu2017structural}. The first three are homogeneous NRL methods, they are widely used in learning representations of homogeneous information networks. Metapath2vec is a metapath-based method designed for learning representations of HINs. DHNE is a recent NRL method using hyperedges to model the relations among multiple objects.

\subsection{Parameter Settings}
In the experiments, the representation size is uniformly set as 64 for all methods. As same as the setting in the previous paper (\cite{tu2017structural}), for Deepwalk and node2vec, we set the window size, walk length and the number of walks on each vertex as 10, 40, and 10, respectively. For LINE, the number of negative samples is set as 5 and the learning rate as 0.025. For metapath2vec, we follow the suggestions from the papers \cite{dong2017metapath2vec,zheng2017recommendation}, the metapaths we chose for DBLP are "APA" and "APCPA", for Douban and IMBD are "MUM", "MAM" and "MDM", for Yelp are "BUB", "BLB" and "BCB". For DHNE, following the setting on paper \cite{tu2017structural} we use one-layer full connection layer to learn tuplewise similarity function, the size of hidden layer is set as 64 and the size of fully connection layer is set as the sum of the embedding length from all types, 256. The parameter $\alpha$ in DHNE is tuned by grid search from \{0.01, 0.1, 1, 2, 5, 10\} and the learning rate is set as 0.025.

For the Event2vec, we use one hidden layer autoencoder for all experiments, the size of the hidden layer is 64 which equals the representation size. The parameter $\beta$ is set as 30, 100, 2, 80 for DBLP, Douban, IMDB, and Yelp respectively. The learning rate is set as 0.025 for all experiments. The event identifier for DBLP, Douban, IMDB, and Yelp are set as paper, movie, movie, and business, respectively.

\subsection{Network Reconstruction}
The network reconstruction task can be used to evaluate the performance of preserving original network structural information. In this section, the embeddings of objects obtained by the embedding methods are used to predict the links of the original networks, specifically, we predict the links among objects based on cosine similarity of their embeddings. The evaluation metrics used in network reconstruction task is AUC \cite{fawcett2006introduction}. We independently run each experiment ten times and present the average performance of network reconstruction on all four datasets in Table \ref{table2}.  The standard deviation is less than 0.015 for all experiments.

\begin{table}
	\centering
	\caption{{\small AUC of network reconstruction.}}\label{table2}
	\begin{tabular}{|c|c|c|c|c|}
		\toprule
		Datasets & DBLP & Douban & IMDB & Yelp \\
		\midrule
		Event2vec & \textbf{0.982} & 0.872 & \textbf{0.987} & 0.924 \\
		DeepWalk & 0.964 & \textbf{0.930} & 0.975 & \textbf{0.973} \\
		node2vec & 0.946 & 0.912 & 0.929 & 0.955 \\
		LINE & 0.535 & 0.644 & 0.564 & 0.514 \\
		metapath2vec & 0.884 & 0.891 & 0.846 & 0.738 \\
		DHNE & 0.503 & 0.714 & 0.739 & 0.549 \\
		\bottomrule
	\end{tabular}
\end{table}

We have the following observations from the results:
\begin{itemize}
	\item Event2vec obtains the best performance on DBLP and IMDB. The AUC values of Douban and Yelp obtained by Event2vec are 0.872 and 0.973, respectively. Overall, the results demonstrate that Event2vec can preserve the original network structure information effectively.
	\item DHNE and Event2vec are both hyperedge-based NRL methods. However, Event2vec outperforms DHNE on all four datasets which demonstrate Event2vec is more effective to preserve the pair-wise based structural information of HINs via preserving the EFP and ESP.
	\item Deepwalk and node2vec obtain good performance on all four datasets and perform better than Event2vec on Douban and Yelp. They can preserve the structural information of HINs effectively. However, they do not excel at preserving the semantic information, as shown in both the below link prediction and node classification tasks.
\end{itemize}

\subsection{Link Prediction}
Link prediction can be used to evaluate the performance of NRL algorithms on capturing the implicit relevance of objects. The better performance on link prediction, the more effective NRL algorithm is. We present the link prediction task on object embeddings obtained on four datasets by all NRL algorithms. Specifically, we predict the links among objects based on the cosine similarity of their embeddings. The evaluation metrics used in this task is AUC.

We randomly split the edges of the original HIN for training and testing. The training set contains 80\% edges of the original network and the testing set contains the left 20\%. Each experiment is independently run 10 times and the average performances on the testing set are shown in Table \ref{table3}. The standard deviation is less than 0.015 for all experiments.

\begin{table}
	\centering
	\caption{{\small AUC of link prediction.}}\label{table3}
	\begin{tabular}{|c|c|c|c|c|}
		\toprule
		Methods & DBLP & Douban & IMDB & Yelp \\
		\midrule
		Event2vec & \textbf{0.901} & \textbf{0.823} & 0.894 & \textbf{0.862} \\
		DeepWalk & 0.794 & 0.677 & 0.839 & 0.841 \\
		node2vec & 0.709 & 0.618 & 0.652 & 0.783 \\
		LINE & 0.697 & 0.710 & 0.748 & 0.531 \\
		metapath2vec & 0.551 & 0.589 & \textbf{0.909} & 0.616 \\
		DHNE & 0.632 & 0.761 & 0.811 & 0.546 \\
		\bottomrule
	\end{tabular}
\end{table}

From the results, we have the following observations:
\begin{itemize}
	\item Event2vec significantly outperforms all baselines on DBLP, Douban, and Yelp, and obtains competitive performance against metapath2vec on IMDB. Event2vec is effective to capture the implicit object relevance.
	\item Event2vec performs better than DeepWalk and node2vec for link prediction on all datasets. The previous task, network reconstruction, has shown DeepWalk and node2vec can work effectively for preserving the structural information of HINs, while the results of link prediction demonstrate that their weakness for preserving the semantic information. As a contrary, it shows that Event2vec is able to effectively preserve both the structural and semantic information.
	\item DHNE and Event2vec both consider the relationships among multiple objects as a whole. The difference between them is that Event2vec considers both the quantities and properties of relations during the representation learning process, while DHNE considers only the former. However, Event2vec significantly outperforms DHNE on all four datasets. It demonstrates the properties of relations can facilitate capturing the implicit semantic information.
\end{itemize}

Furthermore, we repeat the link prediction task on DBLP with different sparsity. The ratio of remaining edges are varied from 10\% to 90\% with a step 10\%, and the rest is used to form the testing set. Each experiment is independently conducted ten times and the average performance is shown in Figure \ref{figure6}. We can observe that Event2vec obtains the best performance on all sparse networks. Event2vec can work effectively on sparse networks.

\begin{figure}
	\centering
	\includegraphics[scale=0.35]{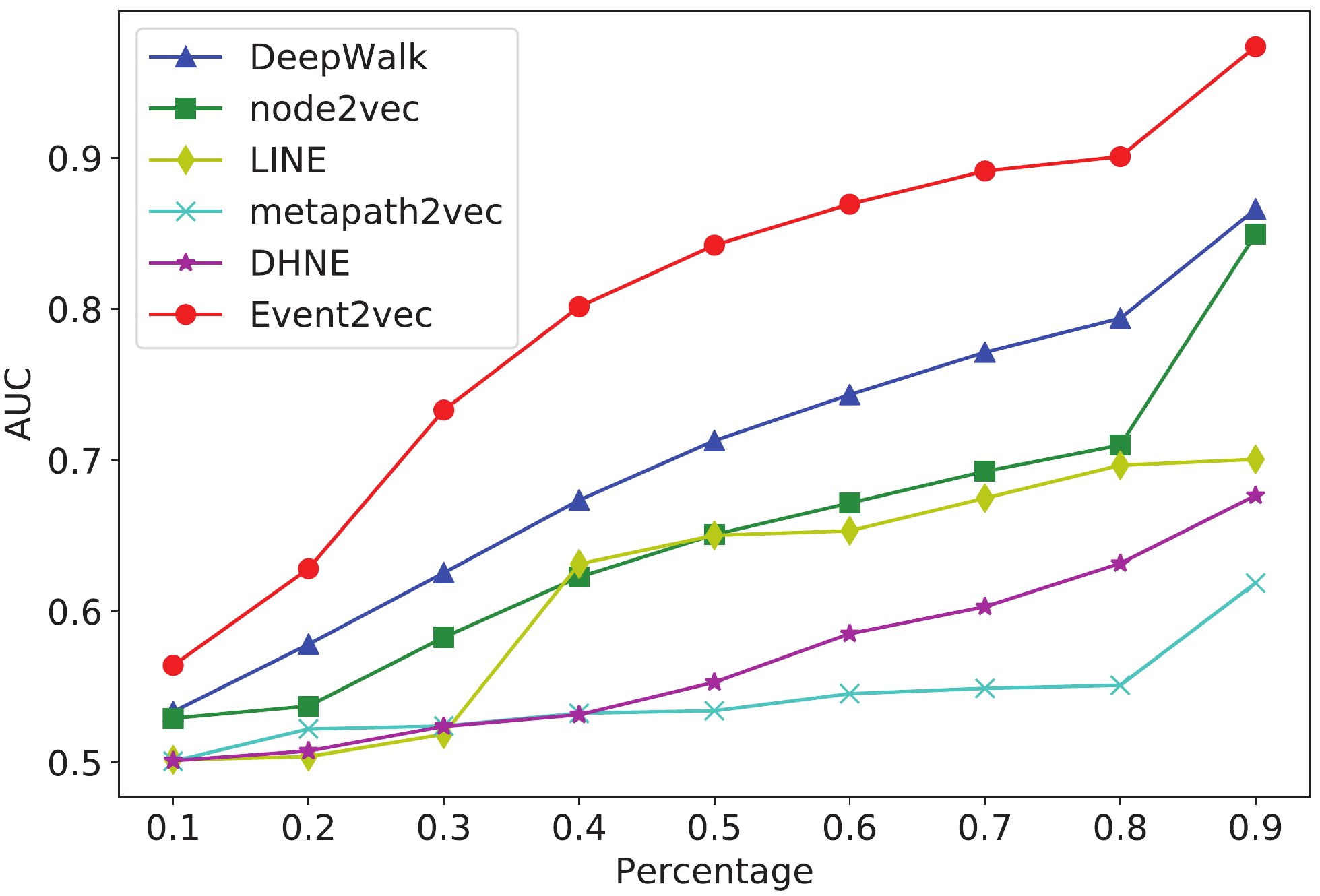}
	\caption{{\small Link prediction results on DBLP of different sparsity.}}\label{figure6}
\end{figure}

\subsection{Node Classification}
Node classification in network analysis is an important task in many applications such as user classification in social networks, movie classification in movie-user networks, and so on. In this section, we conduct the node classification tasks on DBLP and IMDB to evaluate the effectiveness of NRL algorithms on preserving original HIN information. The better node classification performance NRL algorithm obtain, the better effectiveness it has.

The embeddings of objects generated from different methods are used as input features to classify the objects, and the classifier used in our experiments is logistic regression. We randomly sample 10\% to 90\% of the labeled objects as the training set and use the rest labeled objects as the testing set. For DBLP, since each labeled object only receives one label, we use the precision and AUC as the evaluation metrics. For IMDB, each labeled movie has at least one label. Therefore, Micro-f1 and Macro-f1 are used as the evaluation metrics. Each experiment is conducted ten times and the average performance is reported in Figure \ref{figure7}. The standard deviation is less than 0.015 for all experiments.

From the experimental results, we can observe that:

\begin{itemize}
	\item Event2vec performs better than all baselines for node classification task on DBLP and IMDB. It demonstrates that Event2vec is effective to preserve the original structural and semantic information of the original HINs.
	\item Event2vec achieves good performance when there are few labeled objects, even outperforming the performance of some baselines obtained on the cases of vast labeled objects. It demonstrates the robustness of the predictable power of object embeddings obtained by Event2vec.
\end{itemize}

\begin{figure}
	\centering
	\includegraphics[scale=0.28]{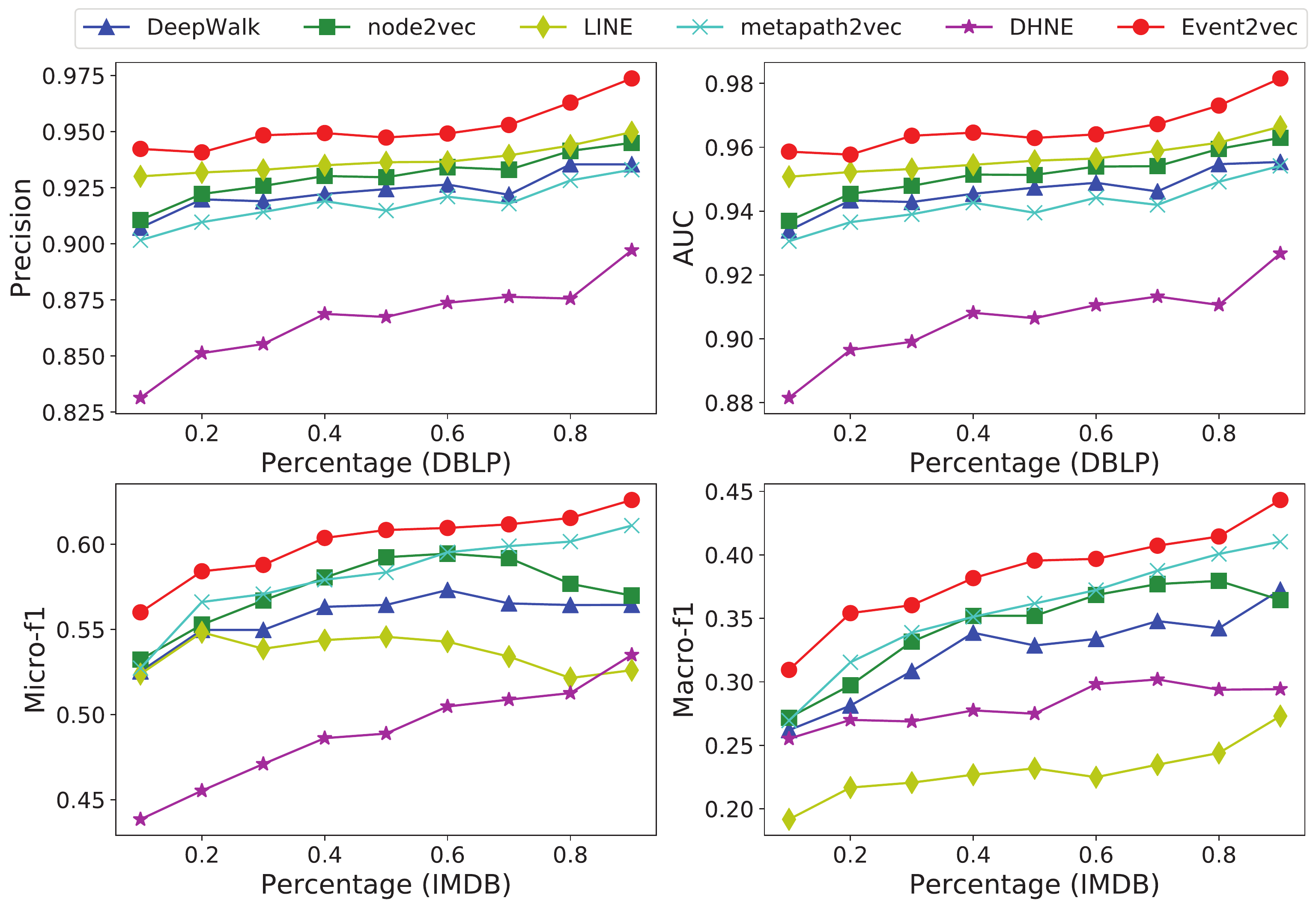}
	\caption{{\small Top: node classification results on DBLP using Precision and AUC as evaluation metrics; bottom: node classification results on IMDB using Micro-f1 and Macro-f1 as evaluation metrics.}}\label{figure7}
\end{figure}

\begin{figure*}
	\centering
	\subfigure[{\small \#dimension $d$}]{
		\includegraphics[scale=0.35]{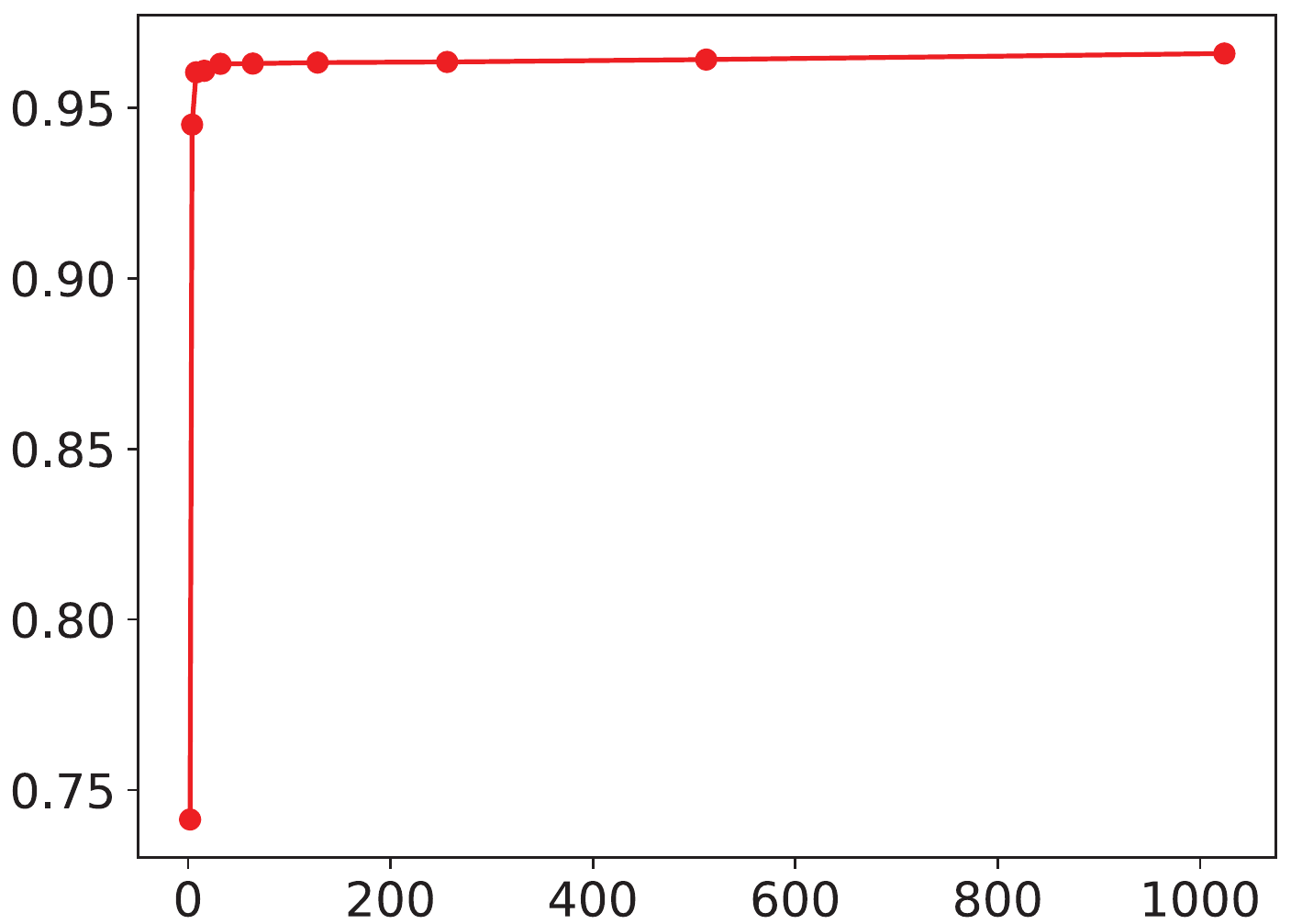}
	}
	\subfigure[{\small penalty rate $\beta$}]{
		\includegraphics[scale=0.35]{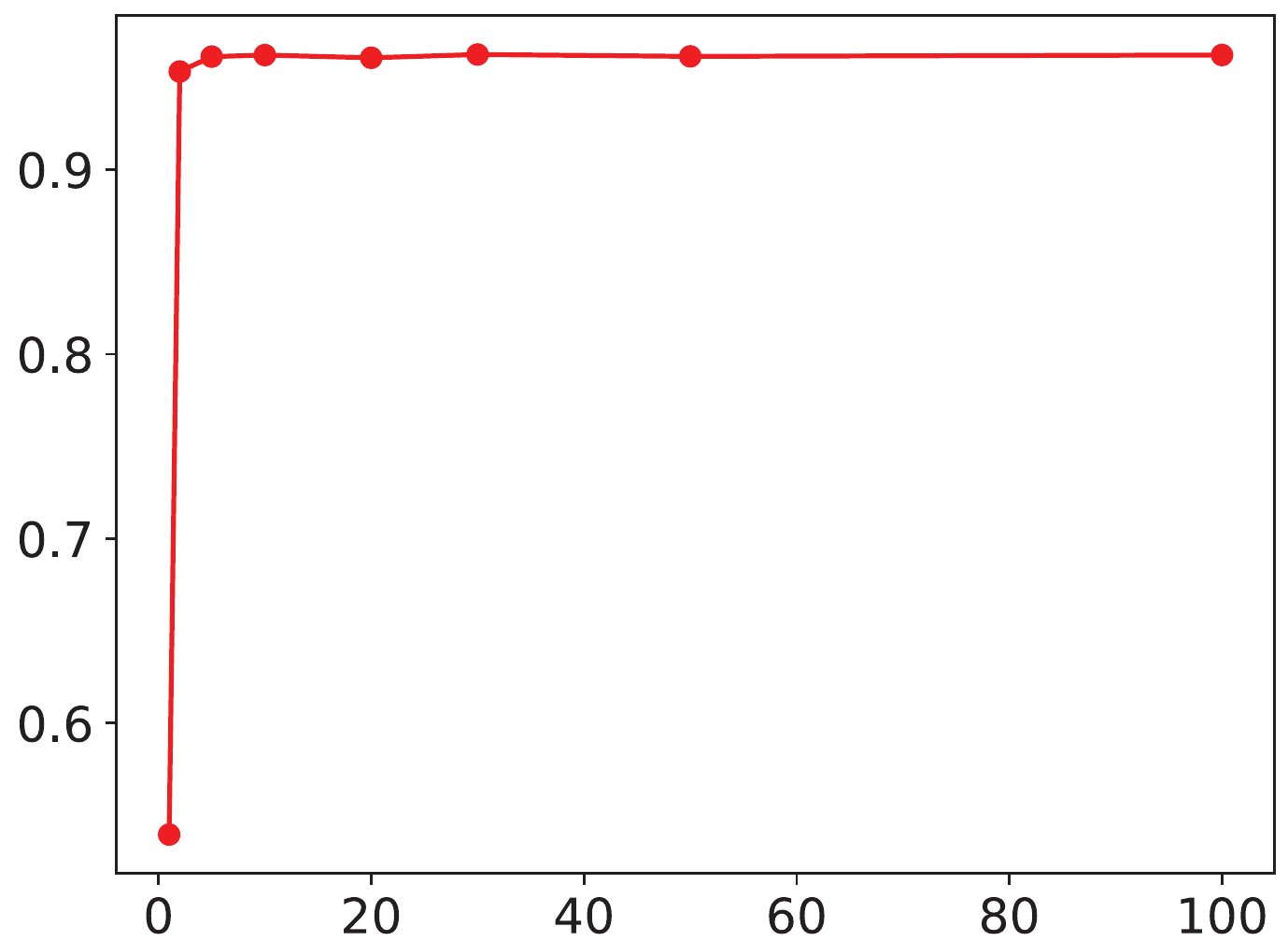}
	}
	\subfigure[{\small \#hidden layers}]{
		\includegraphics[scale=0.35]{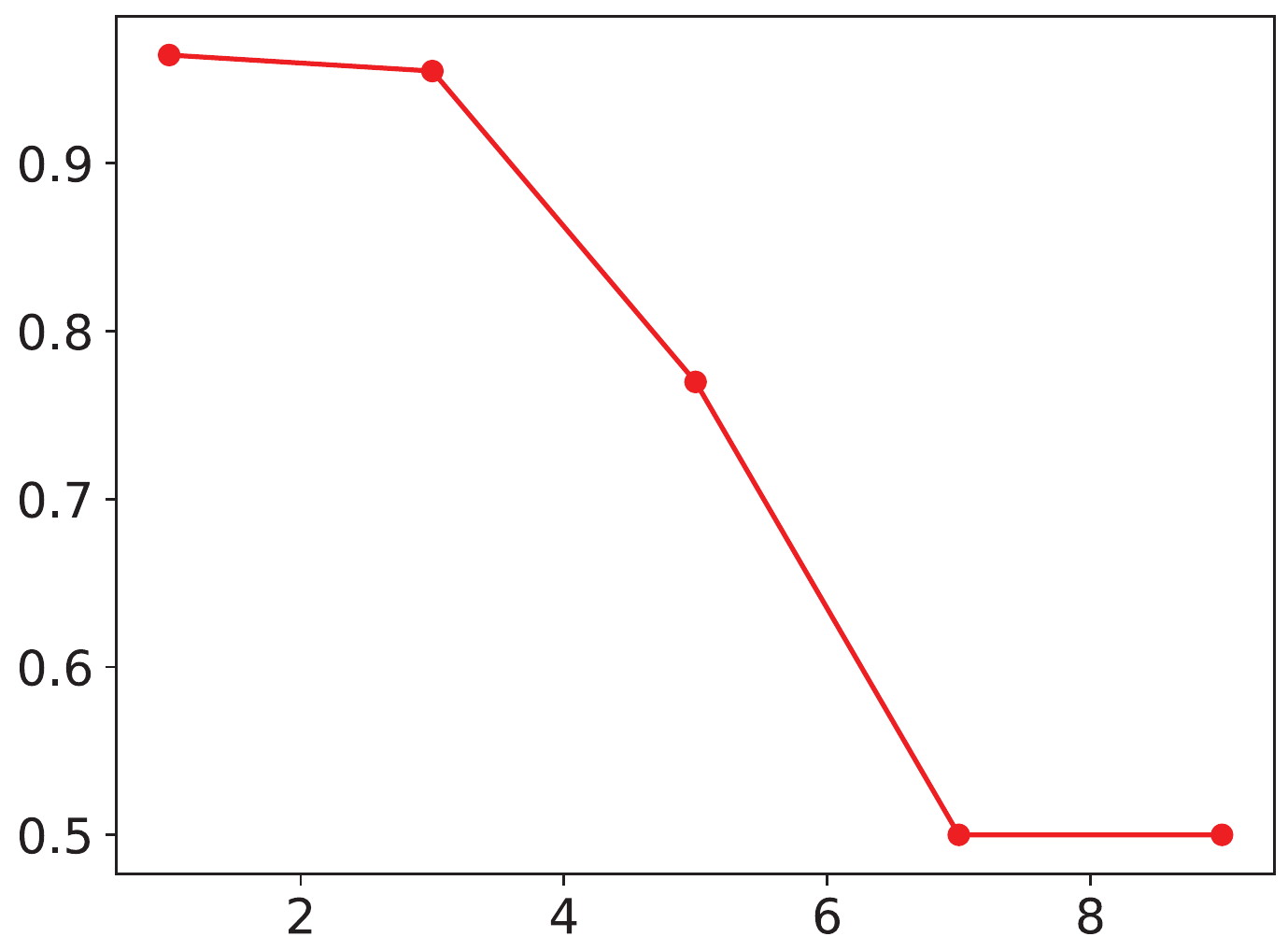}
	}
	\caption{{\small Parameter sensitivity in node classification.}}\label{figure8}
\end{figure*}

\subsection{Parameter Sensitivity Analysis}
We conduct node classification task experiments on DBLP to study the parameter sensitivity in this section. Figure \ref{figure8} shows using 50\% as training set and left as the testing set, the node classification performance (AUC) as a function of one of the chosen parameter when fixing the others.

Figure \ref{figure8}(a) shows that the performance of node classification quickly improves as the number of embedding dimension $d$ increases at the beginning, and become stable when the $d$ is larger than 32. The impact of embedding dimension is reasonable since too small $d$ is inadequate for embodying rich information of HINs. However, when the $d$ is large enough, increasing $d$ will not enrich the network information embodied in the embeddings.

Figure \ref{figure8}(b) shows that when $\beta$ is set as 1, which means the non-zero elements are as important as zero elements of the incident matrix, the performance of node classification is poor. As discussed in section 4, the autoencoder reconstructs the existing link correctly is more important. Therefore, when we set a larger penalty weight to the non-zero elements, the performance of node classification quickly improves and finally becomes stable when the penalty rate is large enough.

With fixing the number of neurons in each layer to 64, we conduct the experiments to study the sensitivity of the number of hidden layers of autoencoder and present the results at Figure \ref{figure8}(c). The results of Figure \ref{figure8}(c) show that the performance of node classification reduces as the number of hidden layers increases and finally as worse as random guessing when the number of hidden layers is larger than 7. Comparing to the single hidden layer autoencoder, deeper autoencoders are easier to over-fit the data. It demonstrates that simple models are sufficient to capture the original structural and semantic information of HINs well in our proposed NRL framework. Furthermore, simple models perform usually more robust to their hyper-parameters and are less prone to over-fitting.

\section{Conclusion}
In this paper, we showed the relevance among multiple objects in HINs should be considered as a whole and such relevance is driven by both the quantities and properties of relations. However, the existing NRL methods consider only the quantities of relations and ignore the impact of their properties. To tackle this issue, we defined the EFP and ESP to measure the object relevance according to the quantities and properties of relations, respectively. A new NRL framework called Event2vec was proposed to learn the object embeddings for HINs, which was theoretically proved that it is able to preserve both the EFP and ESP during the learning process.

To evaluate the performance of the proposed method, we conducted three network analysis tasks on four real-world datasets. The results of network reconstruction showed that Event2vec can effectively preserve the original structural information of HINs. Event2vec significantly outperformed the baselines on link prediction and node classification tasks that demonstrated Event2vec is effective to preserve both structural and semantic information of HINs. DeepWalk and node2vec were shown to work effectively on network reconstruction task as well. However, the results of link prediction and node classification indicated their weakness for preserving the semantic information of HINs. Overall, the experimental results demonstrated the effectiveness of Event2vec for preserving the original structural and semantic information of HINs through preserving the EFP and ESP.

Future work includes exploring more efficient event generating strategies and more powerful models which are used to learn event embeddings in our framework. Since the computational complexity of event generating is $O(|E||\Omega|)$. Designing more efficient event generating strategies can reduce the computational cost and improve the scalability of Event2vec. This paper used the traditional autoencoder to learn event embeddings, while advanced autoencoders such as denoising autoencoders \cite{vincent2010stacked}, variational autoencoders \cite{doersch2016tutorial} or other models may work better for this task. Using more powerful models to learn event embeddings in Event2vec can further improve its performance.